\documentclass[runningheads]{llncs}
\usepackage{amssymb, amsfonts, amsmath}
\usepackage [autostyle, english = american]{csquotes}
\MakeOuterQuote{"}

\newcommand{\F}{\mathcal{F}}
\newcommand{\E}{\mathcal{E}}

\renewcommand{\P}{\mathcal{P}}
\renewcommand{\S}{\mathcal{S}}

\newcommand{\R}{\mathbb{R}}
\newcommand{\B}{\mathcal{B}}

\newcommand{\error}{\varepsilon}
\newcommand{\size}[1]{\vert #1 \vert}

\newcommand{\secref}[1]{Section~\ref{sec:#1}}
\newcommand{\appref}[1]{Appendix~\ref{app:#1}}
\renewcommand{\eqref}[1]{Equation~\ref{eq:#1}}
\newcommand{\thmref}[1]{Theorem~\ref{thm:#1}}

\newcommand{\defref}[1]{Definition~\ref{def:#1}}

\newcommand{\allf}[1]{\mathcal{#1}^{\star, n}}
\newcommand{\allasp}{\mathcal{S}_{ASP}^{\star, n}}
\newcommand{\seq}[1]{({#1}_i)_{i=1}^{k}}

\begin{document}

\title{On the Bounds of Function Approximations}

\author{Adrian de Wynter\inst{1}\orcidID{0000-0003-2679-7241}}
\authorrunning{A. de Wynter}

\institute{$^1$Amazon Alexa, 300 Pine St., Seattle, Washington, USA 98101 \\
  \email{dwynter@amazon.com}
}
\maketitle

\begin{abstract}\footnote{
Citation details:
de Wynter, Adrian. On the Bounds of Function Approximations. In: Tetko, I. V. et al. (eds.) ICANN 2019. 
LNCS, vol 11727. Springer, Heidelberg, pp. 1–17. https://doi.org/10.1007/978-3-030-30487-4{\_}32 \\
The final authenticated publication is available online at https://doi.org/10.1007/978-3-030-30487-4{\_}32
}

Within machine learning, the subfield of Neural Architecture Search (NAS) has recently garnered research attention due to its ability to improve upon human-designed models. However, the computational requirements for finding an exact solution to this problem are often intractable, and the design of the search space still requires manual intervention. In this paper we attempt to establish a formalized framework from which we can better understand the computational bounds of NAS in relation to its search space. For this, we first reformulate the function approximation problem in terms of sequences of functions, and we call it the Function Approximation (FA) problem; then we show that it is computationally infeasible to devise a procedure that solves FA for all functions to zero error, regardless of the search space. We show also that such error will be minimal if a specific class of functions is present in the search space. 
Subsequently, we show that machine learning as a mathematical problem is a solution strategy for FA, albeit not an effective one, and further describe a stronger version of this approach: the Approximate Architectural Search Problem (a-ASP), which is the mathematical equivalent of NAS. We leverage the framework from this paper and results from the literature to describe the conditions under which a-ASP can potentially solve FA as well as an exhaustive search, but in polynomial time.
\end{abstract}

\section{Introduction}
The typical machine learning task can be abstracted out as the problem of finding the set of parameters of a computable function, such that it approximates an underlying probability distribution to seen and unseen examples \cite{Goodfellow-et-al-2016}. Said function is often hand-designed, and the subject of the great majority of current machine learning research. 
It is well-established that the choice of function heavily influences its approximation capability \cite{ArchitecturesBengio,NFLWolpert,XinYao}, and considerable work has gone into automating the process of finding such function for a given task \cite{CarpenterAndGrossberg,Carvalho,Golovin2017GoogleVA}. 
In the context of neural networks, this task is known as Neural Architecture Search (NAS), and it involves searching for the best performing combination of neural network components and parameters from a set, also known as the \emph{search space}. Although promising, little work has been done on the analysis of its viability with respect to its computation-theoretical bounds \cite{Elsken2018NeuralAS}. 
Since NAS strategies tend to be expensive in terms of their hardware requirements \cite{Jin2018AutoKerasEN,Real2017LargeScaleEO}, research emphasis has been placed on optimizing search algorithms, \cite{Elsken2018NeuralAS,MetaDesignOfFFNNs}, even though the search space is still manually designed \cite{Elsken2018NeuralAS,liu2018hierarchical,liu2019darts,Zoph2016NeuralAS}. Without a better understanding of the mathematical confines governing NAS, it is unlikely that these strategies will efficiently solve new problems, or present reliably high performance, thus leading to complex systems that still rely on manually engineering architectures and search spaces. 

Theoretically, learning has been formulated as a function approximation problem where the approximation is done through the optimization of the parameters of a given function \cite{CybenkoSigmoids,Goodfellow-et-al-2016,PoggioTheoryOfNets,PoggioNetworks,Valiant1984ATO}; and with strong results in the area of neural networks in particular \cite{CybenkoSigmoids,FunahashiApprox,HornikApprox2,ShaferRNNs}. On the other hand, NAS is often regarded as a search problem with an optimality criterion \cite{Carvalho,Elsken2018NeuralAS,Real2017LargeScaleEO,SunEtAl,XinYao}, within a given search space. The choice of such search space is critical, yet strongly heuristic \cite{Elsken2018NeuralAS}. Since we aim to obtain a better insight on how the process of finding an optimal architecture can be improved with relation to the search space, we hypothesize that NAS can be enunciated as a function approximation problem. 
The key observation that motivates our work is that all computable functions can be expressed in terms of combinations of members of certain sets, better known as models of computation. Examples of this are the $\mu$-recursive functions, Turing Machines, and, of relevance to this paper, a particular set of neural network architectures \cite{Neto1997TuringUO}. 

Thus, in this study we reformulate the function approximation problem as the task of, for a given search space, finding the procedure that outputs the computable sequence of functions, along with their parameters, that best approximates any given input function. We refer to this reformulation as the Function Approximation (FA) problem, and regard it as a very general computational problem; akin to building a fully automated machine learning pipeline where the user provides a series of tasks, and the algorithm returns trained models for each input.\footnote{Throughout this paper, the problem of data selection is not considered, and is simply assumed to be an input to our solution strategies.} This approach yields promising results in terms of the conditions under which the FA problem has optimal solutions, and about the ability of both machine learning and NAS to solve the FA problem. 

\subsection{Technical Contributions}
The main contribution of this paper is a reformulation of the function approximation problem in terms of sequences of functions, and a framework within the context of the theory of computation to analyze it. Said framework is quite flexible, as it does not rely on a particular model of computation and can be applied to any Turing-equivalent model. We leverage its results, along with well-known results of computer science, to prove that it is not possible to devise a procedure that approximates all functions everywhere to zero error. However, we also show that, if the smallest class of functions along with the operators for the chosen model of computation are present in the search space, it is possible to attain an error that is globally minimal. 

Additionally, we tie said framework to the field of machine learning, and analyze in a formal manner three solution strategies for FA: the Machine Learning (ML) problem, the Architecture Search problem (ASP), and the less-strict version of ASP, the Approximate Architecture Search problem (a-ASP). We analyze the feasibility of all three approaches in terms of the bounds described for FA, and their ability to solve it. In particular, we demonstrate that ML is an ineffective solution strategy for FA, and point out that ASP is the best approach in terms of generalizability, although it is intractable in terms of time complexity. 
Finally, by relating the results from this paper, along with the existing work in the literature, we describe the conditions under which a-ASP is able to solve the FA problem as well as ASP.

\subsection{Outline}
We begin by reviewing the existing literature in \secref{related_work}. In \secref{function_approxs} we introduce FA, and analyze the general properties of this problem in terms of its search space. Then, in \secref{asp_intro} we relate the framework to machine learning as a mathematical problem, and show that it is a weak solution strategy for FA, before defining a stronger approach (ASP) and its computationally tractable version (a-ASP). We conclude in \secref{conclusions} with a discussion of our work. 

\section{Related Work}\label{sec:related_work}
The problem of approximating functions and its relation to neural networks can be found formulated explicitly in \cite{PoggioNetworks}, and it is also mentioned often when defining machine learning as a task, for example in \cite{Bartlett,bendavidzfc,ArchitecturesBengio,Goodfellow-et-al-2016,Valiant1984ATO}. However, it is defined as a parameter optimization problem for a predetermined function. This perspective is also covered in our paper, yet it is much closer to the ML approach than to FA. For FA, as defined in this paper, it is central to find the sequence of functions which minimizes the approximation error. 

Neural networks as function approximators are well understood, and there is a trove of literature available on the subject. An inexhaustive list of examples are the studies found in \cite{CybenkoSigmoids,FunahashiApprox,HornikApprox2,HornikApprox,LeshnoMLP,ParkAndSandberg,pmlr-v80-pham18a,PoggioNetworks,ShaferRNNs,SiegelAndXu,SunEtAl}. It is important to point out that the objective of this paper is not to prove that neural networks are function approximators, but rather to provide a theoretical framework from which to understand NAS in the contexts of machine learning, and computation in general. However, neural networks were shown to be Turing-equivalent in \cite{Neto1997TuringUO,Siegelmann1991TuringCW,Siegelmann1995OnTC}, and thus they are extremely relevant this study. 

NAS as a metaheuristic is also well-explored in the literature, and its application to deep learning has been booming lately thanks to the widespread availability of powerful computers, and interest in end-to-end machine learning pipelines. There is, however, a long standing body of research on this area, and the list of works presented here is by no means complete. Some papers that deal with NAS in an applied fashion are the works found in \cite{AngelineetAl,CarpenterAndGrossberg,Carvalho,Luo2018NeuralAO,SchafferCaruana,stanley:naturemi19,StanleyEvolvingNNs,NIPS1988_149}, while explorations in a formal fashion of NAS and metaheuristics in general can also be found in \cite{Baxter,Carvalho,SiegelAndXu,MetaOptimizationAlgoAnalysis,XinYao}. 
There is also interest on the problem of creating an end-to-end machine learning pipeline, also known as AutoML. Some examples are studies such as the ones in \cite{NIPS2015_5872,he2018amc,Jin2018AutoKerasEN,Wong:2018:TLN:3327757.3327928}. The FA problem is similar to AutoML, but it does not include the data preprocessing step commonly associated with such systems. Additionally, the formal analysis of NAS tends to be as a search, rather than a function approximation, problem. 

The complexity theory of learning and neural networks has been explored as well. The reader is referred to the recent survey from \cite{ComplexityTheoryNNs}, and \cite{Bartlett,BlumerLearnabilityVC,CybenkoTheory,RegularizationAndANN,Vapnik1995TheNO}. 
Leveraging the group-like structure of models of computation is done in \cite{RabinComputability}, and the Blum Axioms \cite{blum} are a well-known framework for the theory of computation in a model-agnostic setting. 
It was also shown in \cite{Bshouty} that, under certain conditions, it is possible to compose some learning algorithms to obtain more complex procedures. Bounds in terms of the generalization error was proven for convolutional neural networks in \cite{cnnsgoogle}. None of the papers mentioned, however, apply directly to FA and NAS in a setting agnostic to models of computation, and the key insights of our work, drawn from the analysis of FA and its solution strategies, are, to the best of our knowledge, not covered in the literature. 
Finally, the Probably Approximately Correct (PAC) learning framework \cite{Valiant1984ATO} is a powerful theory for the study of learning problems. It is a slightly different problem than FA, as the former has the search space abstracted out, while the latter concerns itself with finding a sequence that minimizes the error, by searching through combinations of explicitly defined members of the search space.

\section{A Formulation of the Function Approximation Problem}\label{sec:function_approxs}

In this section we define the FA problem as a mathematical task whose goal is--informally--to find a sequence of functions whose behavior is closest to an input function. We then perform a short analysis of the computational bounds of FA, and show that it is computationally infeasible to design a solution strategy that approximates all functions everywhere to zero error. 

\subsection{Preliminaries on Notation}

Let $\mathbf{R}$ be the set of all total computable functions. Across this paper we will refer to the finite set of \emph{elementary functions} $\E = \{\psi^{1}, ..., \psi^{m} \}$ as the smallest class of functions, along with their operators, of some Turing-equivalent model of computation. 

Let $S = \{\phi_{j} \colon dom(\phi_j) \rightarrow img(\phi_j) \}_{{j \in J}}$ be a set of functions defined over some sets $dom(\phi_j),img(\phi_j)$, such that $S$ is indexed by a set $J$, and that $S \subset \mathbf{R}$. Also let $f(x) = (\phi_{i_1}, \phi_{i_2}, ..., \phi_{i_k})(x)$ be a sequence of elements of $S$ applied successively and such that $i_1, ..., i_k \in I$ for some $I \subset J$. 
We will utilize the abbreviated notation $f =\seq{\phi}$ to denote such a sequence; and we will use $\allf{S} = \{ \seq{\phi} \vert \; \phi_{i} \in S, k\leq n\}$ to describe the set of all $n$-or-less long possible sequences of functions drawn from said $S$, such that $f \in \allf{S} \Leftrightarrow f \in \mathbf{R}$.

For consistency purposes, throughout this paper we will be using Zermelo-Fraenkel with the Axiom of Choice (ZFC) set theory. Finally, for simplicity of our analysis we will only consider continuous, real-valued functions, and beginning in \secref{search_spaces}, only computable functions.

\subsection{The FA Problem}

Prior to formally defining the FA problem, we must be able to quantify the behavioral similarity of two functions. This is done through the \emph{approximation error} of a function:

\begin{definition}[The approximation error]\label{def:approx_error}

Let $f$ and $g$ be two functions. Given a nonempty subset $\sigma \subset dom(g)$,
the \emph{approximation error} of a function $f$ to a function $g$ is a procedure which outputs $0$ if $f$ is equal to $g$ with respect to some metric $d \colon \R \times \R \rightarrow \R_{\geq 0}$ across all of $\sigma$, and a positive number otherwise:
\begin{equation}
\error_{\sigma}(f, g) = \frac{1}{\size{\sigma}}\sum\limits_{x \in \sigma} d(f(x), g(x))
\end{equation}

Where we assume that, for the case where $x \not\in dom(f)$, $d(f(x), g(x)) = g(x)$. 

\end{definition}

\begin{definition}[The FA Problem]\label{def:fa_problem}
For any input function $\F$, given a function set (the \emph{search space}) $\S$, an integer $n \in \mathbb{N}_{>0}$, and nonempty sets $\sigma \subset dom(\F)$, find the sequence of functions $f = \seq{\phi},\; \phi_{i} \in \S,\; k \leq n$, such that $\error_{\sigma}(f, \F)$ is minimal among all members of $\allf{S}$ and $\sigma$.
\end{definition}

The FA problem, as stated in \defref{fa_problem}, makes no assumptions regarding the characterization of the search space, and follows closely the definition in terms of optimization of parameters from \cite{PoggioTheoryOfNets,PoggioNetworks}. However, it makes a point on the fact that the approximation of a function should be given by a sequence of functions. 

If the input function were to be continuous and multivariate, we know from \cite{Kolmogorov,Ostrand} that there exists at least one exact (i.e., zero approximation error) representation in terms of a sequence of single-variable, continuous functions. If such single-variable, continuous functions were to be present in $\S$, one would expect that the FA problem could solved to zero error for all continuous multivariate inputs, by simply comparing and returning the right representation.\footnote{With the possible exception of the results from \cite{Vitushkin}.}
However, it is infeasible to devise a generalized algorithmic procedure that outputs such representation:

\begin{theorem}\label{thm:error_zero}

There is no computable procedure for FA that approximates all continuous, real-valued functions to zero error, across their entire domain.

\end{theorem}
\begin{proof}
Solution strategies for FA are parametrized by the sequence length $n$, the subset of the domain $\sigma$, and the search space $\S$. 

Assume $\S$ is infinite. 
The input function $\F$ may be either computable or uncomputable. 
If the input $\F$ is uncomputable, by definition it can only be estimated to within its computable range, and hence its approximation error is nonzero. 
If $\F$ is a computable function, we have guaranteed the existence of at least one function within $\allf{S}$ which has zero approximation error: $\F$ itself. 
Nonetheless, determining the existence of such a function is an undecidable problem. To show this, it suffices to note that it reduces to the problem of determining the equivalence of two halting Turing Machines by asking whether they accept the same language, which is undecidable. 

When $n$ or $\sigma$ are infinite, there is no guarantee that a procedure solving FA will terminate for all inputs. 

When $n$, $\sigma$, or $\S$ are finite, there will always be functions outside of the scope of the procedure that can only be approximated to a nonzero error.

Therefore, there cannot be a procedure for FA that approximates all functions, let alone all computable functions, to zero error for their entire domain.
\qed
\end{proof}

It is a well-known result of computer science that neural networks \cite{CybenkoSigmoids,FunahashiApprox,Goodfellow-et-al-2016,HornikApprox2,HornikApprox}, and PAC learning algorithms \cite{Valiant1984ATO}, are able to approximate a large class of functions to an arbitrary, non-zero error. However, \thmref{error_zero} does not make any assumptions regarding the model of computation used, and thus it works as more generalized statement of these results. 

For the rest of this paper we will limit ourselves to the case where $n$, $\sigma$, and $\S$ are finite, and the elements of $\S$ are computable functions. 

\subsection{A Brief Analysis of the Search Space}\label{sec:search_spaces}

It has been shown that the solutions to FA can only be found in terms of finite sequences built from a finite search space, whose error with respect to the input function is nonzero. 
It is worth analyzing under which conditions these sequences will present the smallest possible error. 

For this, we note that any solution strategy for FA will have to first construct at least one sequence $f \in \allf{S}$, and then compute its error against the input function $\F$. It could be argued that this "bottom-up" approach is not the most efficient, and one could attempt to "factor" a function in a given model of computation that has explicit reduction formulas, such as the Lambda calculus. This, unfortunately, is not possible, as the problem of determining the reduction of a function in terms of its elementary functions is well-known to be undecidable \cite{ChurchAnUP}. 
However, the idea of "factoring" a function can still be leveraged to show that, if the set of elementary functions $\E$ is present in the search space $\S$, any sufficiently clever procedure will be able to get the smallest possible theoretical error for $\S$, for any given input function $\F$:

\begin{theorem}\label{thm:lower_bound}

Let $\S$ be a search space such that it contains the set of elementary functions, $\E \subset \S$. Then, for any input function $\F$, there exists at least one sequence $f_o \in \allf{S}$ with the smallest approximation error among all possible computable functions of sequence length up to and including $n$. 
\end{theorem}
\begin{proof}
By definition, $\E$ can generate all possible computable functions. If $\E \not\subset \S$, then $\size{\allf{S}} < \size{\allf{E}}$, and so there exist input functions whose sequence with the smallest approximation error, $f_o$, is not contained in $\allf{S}$. 
\qed
\end{proof}

In practice, constructing a space that contains $\E$, and subsequently performing a search over it, can become a time consuming task given that the number of possible members of $\allf{S}$ grows exponentially with $n$. On the other hand, constructing a more "efficient" space that already contains the best possible sequence requires prior knowledge of the structure of a function relating $\S$ to $\F$--the problem that we are trying to solve in the first place. 
That being said, \thmref{lower_bound} implies that there must be a way to quantify the ability of a search space to generalize to any given function, without the need of explicitly including $\E$. To achieve this, we first look at the ability of every sequence to approximate a function, by defining the \emph{information capacity} of a sequence:

\begin{definition}[The Information Capacity]\label{def:info_cap}
Let $f = (\phi_i)_{i=1}^{n}$ be a finite sequence, where every $\phi_i$ has associated a finite set of possible parameters $\pi_i$, and a restriction set $\rho_i$ in its domain: $\phi_i \colon dom(\phi_i) \times \pi_i \rightarrow img(\phi_i) \setminus \rho_i$, so that the next element in the sequence is a function $\phi_{i+1}$ with $dom(\phi_{i+1}) = img(\phi_i) \setminus \rho_i$. 

Then the \emph{information capacity} of a sequence $f$ is given by the Cartesian product of the domain, parameters, and range of each $\phi_{i}$:

\begin{equation}
C(f) = dom(\phi_1) \times \big( \prod\limits_{i = 1}^{n - 1} \pi_i \times (img(\phi_i) \setminus \rho_i) \big) \times \pi_n \times img(\phi_n)
\end{equation}
\end{definition}

Note that the information capacity of a function is quite similar to its graph, but it makes an explicit relationship with its parameters. Specifically, in the case where $\pi_i \subset \Pi$ for every $\pi_i$ in some $f$, $C(f) = dom(\phi_1) \times \Pi \times img(\phi_n)$.

At a first glance, \defref{info_cap} could be seen as a variant of the VC dimension \cite{BlumerLearnabilityVC,Vapnik1995TheNO}, since both quantities attempt to measure the ability of a given function to generalize. However, the latter is designed to work on a fixed function, and our focus is on the problem of building such a function. A more in-depth discussion of this distinction, along with its application to the framework from this paper, is given in \secref{ml_context}, and in \appref{vc_dim}.

A search space is comprised of one or more functions, and algorithmically we are more interested about the quantifiable ability of the search space to approximate any input function. 
Therefore, we define the \emph{information potential} of a search space as follows:

\begin{definition}[The Information Potential]\label{def:info_pot}
The \emph{information potential of a search space} $\S$, is given by all the possible values its members can take for a given sequence length $n$:
\begin{equation}
U(\S, n) = \bigcup\limits_{f \in \allf{S}} C(f)
\end{equation}
\end{definition}

The definition of the information potential allows us to make the important distinction between comparing two search spaces $S_1, S_2$ containing the same function $f$, but defined over different parameters $\pi_1, \pi_2 \subset \Pi$; and comparing $S_1$ and $S_2$ with another space, $S_3$, containing a different function $g$: the information potentials will be equivalent on the first case, $U(S_1, n) = U(S_2, n)$, but not on the second: $U(S_3, n) \neq U(S_1, n)$.

For a given space $\S$, as the sequence length $n$ grows to infinity, and if the search space includes the set of elementary functions, $\E \subset \S$, its information potential encompasses all computable functions:
\begin{equation}\label{eq:unlimited_potential}
\lim\limits_{n\rightarrow \infty} U(\S, n) = \mathbf{R}
\end{equation}
In other words, the information potential of such $\S$ approaches the information capacity of a universal approximator, which depending on the model of computation chosen, might be a universal Turing machine, or the universal function from \cite{RogersComputability}, to name a few.

In the next section, we leverage the results shown so far to evaluate three different procedures to solve FA, and show that there exists a best possible solution strategy. 


\section{The FA Problem in the Context of Machine Learning}\label{sec:asp_intro}

In this section we relate the results from analyzing FA to the field of machine learning. First, we show that the machine learning task can be seen as a solution strategy for FA. We then introduce the Architecture Search Problem (ASP) as a theoretical procedure, and note that it is the best possible solution strategy for FA. Finally, we note that ASP is unviable in an applied setting, and define a more relaxed version of this approach: the Approximate Architecture Search Problem (a-ASP), which is the analogous of the NAS task commonly seen in the literature. 

\subsection{Machine Learning as a Solver for FA}\label{sec:ml_context}

The Machine Learning (ML) problem, informally, is the task of approximating an input function $\F$ through repeated sampling and the parameter search of a predetermined function. This definition is a simplified, abstracted out version of the typical machine learning task. It is, however, not new, and a brief search in the literature (\cite{bendavidzfc,ArchitecturesBengio,Goodfellow-et-al-2016,PoggioTheoryOfNets}) can attest to the existence of several equivalent formulations. We reproduce it here for notational purposes, and constrain it to computable functions:

\begin{definition}[The ML Problem]\label{def:ml_problem}
For an unknown, continuous function $\F$ defined over some domain $dom(\F)$, given finite subsets $\sigma \subset dom(\F)$, a function $f$ with parameters from some finite set $\Pi$, and a function $m \colon \R \times \R \rightarrow \R_{\geq 0}$, find a $\pi_o \in \Pi$ such that $m(f(x, \pi_o), \F(x))$ is minimal for all $x\in \sigma$.
\end{definition}

As defined in \defref{fa_problem}, any procedure solving FA is required to return the sequence that best approximates any given function. In the ML problem, however, such sequence $f$ is already given to us. Even so, we can still reformulate ML as a solution strategy for FA.
For this, let the search space be a singleton of the form $\S_{ML} = \{ f \}$; set $m$ to be the metric function $d$ in the approximation error; and leave $\sigma$ as it is. We then carry out a "search" over this space by simply picking $f$, and then optimizing the parameters of $f$ with respect to the approximation error $\error_\sigma(f, \F)$. We then return the function along with the parameters $\pi_o$ that minimize the error.

Given that the search is performed over a single element of the search space, this is not an effective procedure in terms of generalizability. 
To see this, note that the procedure acts as intended, and "finds" the function that minimizes the approximation error $\error_\sigma(f, \F)$ between $f$ and any other $\F$ in the search space $\S_{ML}$. 
However, being able to approximate an input function $\F$ in a single-element search space tells us nothing about the ability of ML to approximate other input functions, or even whether such $f \in \S_{ML}$ is the best function approximation for $\F$ in the first place. In fact, we know by \thmref{lower_bound} that for a given sequence length $n$, for every $\F$ there exists an optimal sequence $f_o$ in $\allf{E}$, which is may not be present in $\S_{ML}$.

Since we are constrained to a singleton search space, one could be tempted to build a search space with one single function that maximizes the information potential, such as the one as described in \eqref{unlimited_potential}, say, by choosing $f$ to be a universal Turing Machine. There is one problem with this approach: this would mean that we need to take in as an input the encoding of the input function $\F$, along with the subset of the domain $\sigma$. If we were able to take the encoding of $\F$ as part of the input, we would already know the function and this would not be a function approximation problem in the first place. Additionally, 
we would only be able to evaluate the set of computable functions which take in as an argument their own encoding, as it, by definition, needs to be present in $\sigma$.

In terms of the framework from this paper we can see that, no matter how we optimize the parameters of $f$ to fit new input functions, the information potential $U(\S_{ML}, n)$ remains unchanged, and the error will remain bounded. 
This leads us to conclude that measuring a function's ability to learn through its number of parameters \cite{Goodfellow-et-al-2016,SontagVC,Vapnik1995TheNO} is a good approach for a fixed $f$ and single input $\F$, but incomplete in terms of describing its ability to generalize to other problems. 
This is of critical importance, because, in an applied setting, even though nobody would attempt to use the same architecture for all possible learning problems, the choice of $f$ remains a crucial, and mostly heuristic, step in the machine learning pipeline.

The statements regarding the information potential of the search space are in accordance with the results in \cite{NFLWolpert}, where it was shown that--in the terminology of this paper--two predetermined sequences $f$ and $f'$, when averaging their approximation error across all possible input functions, will have equivalent performance.
We have seen that ML is unable to generalize well to any other possible input function, and is unable to determine whether the given sequence $f$ is the best for the given input.
This leads us to conclude that, although ML is a computationally tractable solution strategy for FA, it is a weak approach in terms of generalizability. 

\subsection{The Architecture Search Problem (ASP)}

We have shown that ML is a solution strategy for FA, although the nature of its search space makes it ineffective in a generalized setting. It is only natural to assume that a stronger formulation of a procedure to solve FA would involve a more complex search space. 

Similar to \defref{ml_problem}, we are given the task of approximating an unknown function $\F$ through repeated sampling. Unlike ML, however, we are now able to select the sequence of functions (i.e., architecture) that best fits a given input function $\F$:
\begin{definition}[The Architecture Search Problem (ASP)]\label{def:asp}
For an unknown, continuous function $\F$ defined over some domain $dom(\F)$, given a finite subset $\sigma \subset dom(\F)$, a sequence length $n$, a search space $\S_{ASP}$, and a function $m \colon \R \times \R \rightarrow \R_{\geq 0}$, find the sequence $f = \seq{\phi},\; \phi_{i} \in \S_{ASP}$, $k \leq n$ such that $m(f(x), \F(x))$ is minimal for all $x\in \sigma$, and all $f \in \allasp$. 
\end{definition}

Note that we have left the parameter optimization problem implicit in this formulation, since, as pointed out in \secref{ml_context}, a single-function search space $f$ would be ineffective for dealing with multiple input functions $\F$, no matter how well the optimizer performed for a given subset of these inputs. 

At a first glance, ASP looks similar to the PAC learning framework \cite{Valiant1984ATO}. However, FA is the task about finding the right sequence of computable functions for all possible functions, while PAC is a generalized, tractable formulation of learning problems, with the search space abstracted out. A more precise analysis of the relationship between FA and PAC is described in \appref{asp_pac}. 

As a solution strategy for FA, ASP is also subject to the results from section \secref{function_approxs}. The key difference between ML and ASP is that ASP has access to a richer search space, which allows it to have a better approximation capability. 
In particular, ASP could be seen as a generalized version of the former, since for any $n$-sized sequence present in $\S_{ML}$, one could construct a space with bigger information potential in ASP, but with the same constrains in sequence length. For example, we could use $\E$ as our search space, choose a sequence length $n$, and so $U(\S_{ML}, n) \subset U(\E, n)$. 

Since ASP has no explicit constraints on time and space, this procedure is essentially performing an exhaustive search. \thmref{lower_bound} implies that, for fixed $n$ and any input $\F$, ASP will always return the best possible sequence within that space, as long as the search space contains the set of elementary functions, $\E \subset \S$.

On the other hand, it is a cornerstone of the theory and practice of machine learning that learning algorithms must be tractable--that is, they must run in polynomial time. Given that the search space for ASP grows exponentially with the sequence length, this approach is an interesting theoretical tool, but not very practical. We will still use ASP as a performance target for the evaluation of more applicable procedures. 
However, it is desirable to formulate a solution strategy for FA that can be used in an applied setting, but can also be analyzed within the framework of this paper. 

To achieve this, first we note that any other solution strategy for FA which terminates in polynomial time will have to be able to avoid verifying every possible function in the search space. In other words, such procedure would require a function that is able to choose a nonempty subset of the search space. We denote such function as $\B$, such that for a search space $\S$, $\B(\S) \subset \allf{S}$.
We can now define the Approximate Architecture Search Problem (a-ASP) as the formulation of NAS in terms of the FA framework:

\begin{definition}[The Approximate ASP (a-ASP)]\label{def:a_asp}
If $\F$ is an unknown, continuous function defined over some domain $dom(\F)$, given a finite subset $\sigma \subset dom(\F)$, a sequence length $n$, a search space $\S_{ASP}$, a function $m \colon \R \times \R \rightarrow \R_{\geq 0}$, 
and a set builder function $\mathcal{B}(\S_{ASP}) \subset \allasp$, 
find the sequence $f = \seq{\phi},\; \phi_{i} \in \B(\S_{ASP})$, $k \leq n$ such that 
$m(f(x), \F(x))$ is minimal for all $x \in \sigma$ and $f \in \mathcal{B}(\S_{ASP})$.
\end{definition}

Just as the previous two procedures we defined, a-ASP is also a solution strategy for FA. The only difference between \defref{asp} and \defref{a_asp} is the inclusion of the set builder function to traverse the space in a more efficient manner. 
Due to the inclusion of this function, however, a-ASP is weaker than ASP, since it is not guaranteed to find the functions $f_o$ that globally minimizes $\error_\sigma(f_o, \F)$, for all given $\F$. 
Additionally, the fact that this function must be included into the parameters for a-ASP implies that such procedure requires some design choices. Given that everything else in the definition of a-ASP is equivalent to ASP, it can be stated that the set builder function is the only deciding factor when attempting to match the performance of ASP with a-ASP.

It has been shown \cite{Wolpert2005CoevolutionaryFL} that certain set builder functions perform better than others in a generalized setting. 
This can be also seen from the perspective of the FA framework, where we have available at our disposal the sequences that make up a given function. 
In particular, if $\S = \{\phi_1, ..., \phi_m\}$ is a search space, and $\B$ is a function that selects elements from $\allf{S}$, a-ASP not only has access to the performance of all the $k$ sequences chosen so far, $\{\error_\sigma(f_i, F), f_i \in \B(\allf{S})\}_{i \in \{1, ..., k\}}$, but also the encoding (the configurations from \cite{Wolpert2005CoevolutionaryFL}) of their composition. 
This means that, given enough samples, when testing against a subset of the input, $\sigma' \subset \sigma$, such an algorithm would be able to learn the expected output $\phi(s)$ of the functions $\phi \in \S$, and their behavior if included in the current sequence $f_{k+1} = (f_k, \phi)(s)$, for $s \in \sigma'$. Including such information in a set builder function could allow the procedure to make better decisions at every step, and this approach has been used in applied settings with success \cite{MillerAndHedge,liu2018hierarchical}. 

It can be seen that these design choices are not necessarily problem-dependent, and, from the results of \thmref{lower_bound}, they can be done in a theoretically motivated manner. 
Specifically, we note that the information potential of the search space remains unchanged between a-ASP and ASP, and so, by including $\E$, a-ASP could have the ability to perform as well as ASP.

\section{Conclusion}\label{sec:conclusions}
The FA problem is a reformulation of the problem of approximating any given function, but with finding a sequence of functions as a central aspect of the task. In this paper, we analyzed its properties in terms of the search space, and its applications to machine learning and NAS. In particular, we showed that it is impossible to write a procedure that solves FA for any given function and domain with zero error, but described the conditions under which such error can be minimal. 
We leveraged the results from this paper to analyze three solution strategies for FA: ML, ASP, and a-ASP. 
Specifically, we showed that ML is a weak solution strategy for FA, as it is unable to generalize or determine whether the sequence used is the best fit for the input function. We also pointed out that ASP, although the best possible algorithm to solve for FA, is intractable in an applied setting. 

We finished by formulating a solution strategy that merged the best of both ML and ASP, a-ASP, and pointed out, through existing work in the literature, complemented with the results from this framework, that it has the ability to solve FA as well as ASP in terms of approximation error.

One area that was not discussed in this paper was whether it would be possible to select \emph{a priori} a good subset $\sigma$ of the input function's domain. This problem is important since a good representative of the input will greatly influence a procedure's capability to solve FA. This is tied to the data selection process, and it was not dealt with on this paper. Further research on this topic is likely to bear great influence on machine learning as a whole.

\section*{Acknowledgments}
The author is grateful to the anonymous reviewers for their helpful feedback on this paper, and also thanks Y. Goren, Q. Wang, N. Strom, C. Bejjani, Y. Xu, and B. d'Iverno for their comments and suggestions on the early stages of this project. 

\bibliography{de_wynter}

\begin{thebibliography}{10}
\providecommand{\url}[1]{\texttt{#1}}
\providecommand{\urlprefix}{URL }
\providecommand{\doi}[1]{https://doi.org/#1}

\bibitem{AngelineetAl}
Angeline, P.J., Saunders, G.M., Pollack, J.B.: An evolutionary algorithm that
  constructs recurrent neural networks. Trans. Neur. Netw.  \textbf{5}(1),
  54--65 (1994). \doi{10.1109/72.265960}

\bibitem{Bartlett}
Bartlett, P.L., Ben-David, S.: Hardness results for neural network
  approximation problems. In: Proceedings of the 4th European Conference on
  Computational Learning Theory. pp. 50--62. EuroCOLT '99, Springer-Verlag,
  London, UK, UK (1999). \doi{10.1016/S0304-3975(01)00057-3}

\bibitem{Baxter}
Baxter, J.: A model of inductive bias learning. Journal of Artificial
  Intelligence Research  \textbf{12},  149--198 (2000). \doi{10.1613/jair.731}

\bibitem{bendavidzfc}
Ben{-}David, S., Hrubes, P., Moran, S., Shpilka, A., Yehudayoff, A.: A learning
  problem that is independent of the set theory {ZFC} axioms. CoRR
  \textbf{abs/1711.05195} (2017), \url{http://arxiv.org/abs/1711.05195}

\bibitem{ArchitecturesBengio}
Bengio, Y.: Learning deep architectures for ai. Foundations and Trends in
  Machine Learning  \textbf{2}(1),  1--127 (2009). \doi{10.1561/2200000006}

\bibitem{blum}
Blum, M.: A machine-independent theory of the complexity of recursive
  functions. Journal of the ACM  \textbf{14}(2),  322–336 (1967).
  \doi{10.1145/321386.321395}

\bibitem{BlumerLearnabilityVC}
Blumer, A., Ehrenfeucht, A., Haussler, D., Warmuth, M.K.: Learnability and the
  vapnik-chervonenkis dimension. Journal of the Association for Computing
  Machinery  \textbf{36},  929--965 (1989). \doi{10.1145/76359.76371}

\bibitem{Bshouty}
Bshouty, N.H.: A new composition theorem for learning algorithms. In:
  Proceedings of the Thirtieth Annual ACM Symposium on Theory of Computing. pp.
  583--589. STOC '98, ACM, New York, NY, USA (1998).
  \doi{10.1145/258533.258614}

\bibitem{CarpenterAndGrossberg}
Carpenter, G.A., Grossberg, S.: A massively parallel architecture for a
  self-organizing neural pattern recognition machine. Computer Vision, Graphics
  and Image Processing  \textbf{37},  54--115 (1987).
  \doi{10.1016/S0734-189X(87)80014-2}

\bibitem{Carvalho}
Carvalho, A.R., Ramos, F.M., Chaves, A.A.: Metaheuristics for the feedforward
  artificial neural network (ann) architecture optimization problem. Neural
  Comput. \& Aplic.  (2010). \doi{10.1007/s00521-010-0504-3}

\bibitem{ChurchAnUP}
Church, A.: An unsolvable problem of elementary number theory. American Journal
  of Mathematics  \textbf{58},  345--363 (1936)

\bibitem{CybenkoSigmoids}
Cybenko, G.: Approximation by superpositions of a sigmoidal function.
  Mathematics of Control, Singals, and Systems  \textbf{2},  303--314 (1989).
  \doi{10.1007/BF02551274}

\bibitem{CybenkoTheory}
Cybenko, G.: Complexity theory of neural networks and classification problems.
  In: Proceedings of the EURASIP Workshop 1990 on Neural Networks. pp. 26--44.
  Springer-Verlag (1990). \doi{10.1007/3-540-52255-7{\_}25}

\bibitem{Elsken2018NeuralAS}
Elsken, T., Metzen, J.H., Hutter, F.: Neural architecture search: A survey
  (2019). \doi{10.1007/978-3-030-05318-5{\_}3}

\bibitem{NIPS2015_5872}
Feurer, M., Klein, A., Eggensperger, K., Springenberg, J., Blum, M., Hutter,
  F.: Efficient and robust automated machine learning. In: Cortes, C.,
  Lawrence, N.D., Lee, D.D., Sugiyama, M., Garnett, R. (eds.) Advances in
  Neural Information Processing Systems 28, pp. 2962--2970. Curran Associates,
  Inc. (2015)

\bibitem{FunahashiApprox}
Funahashi, K.: On the approximate realization of continuous mappings by neural
  networks. Neural Networks  \textbf{2},  183--192 (1989).
  \doi{10.1016/0893-6080(89)90003-8}

\bibitem{RegularizationAndANN}
Girosi, F., Jones, M., Poggio, T.: Regularization theory and neural networks
  architectures. Neural Computation  \textbf{7},  219--269 (1995).
  \doi{10.1162/neco.1995.7.2.219}

\bibitem{Golovin2017GoogleVA}
Golovin, D., Solnik, B., Moitra, S., Kochanski, G., Karro, J., Sculley, D.:
  Google vizier: A service for black-box optimization  (2017).
  \doi{10.1145/3097983.3098043}

\bibitem{Goodfellow-et-al-2016}
Goodfellow, I., Bengio, Y., Courville, A.: Deep Learning. MIT Press, Cambridge,
  MA (2016), \url{http://www.deeplearningbook.org}

\bibitem{he2018amc}
He, Y., Lin, J., Liu, Z., Wang, H., Li, L.J., Han, S.: Amc: Automl for model
  compression and acceleration on mobile devices. In: Proceedings of the
  European Conference on Computer Vision (ECCV). pp. 784--800 (2018)

\bibitem{HornikApprox2}
Hornik, K.: Approximation capabilities of multilayer feedforward networks.
  Neural Networks  \textbf{4},  251--257 (1991).
  \doi{10.1016/0893-6080(91)90009-T}

\bibitem{HornikApprox}
Hornik, K., Stinchcombe, M., White, H.: Multilayer feedforward networks are
  universal approximators. Neural Networks  \textbf{2},  359--366 (1989).
  \doi{10.1016/0893-6080(89)90020-8}

\bibitem{Jin2018AutoKerasEN}
Jin, H., Song, Q., Hu, X.: Auto-keras: Efficient neural architecture search
  with network morphism (2018)

\bibitem{Kolmogorov}
Kolmogorov, A.N.: On the representation of continuous functions of several
  variables by superposition of continuous function of one variable and
  addition. Dokl. Akad. Nauk SSSR  \textbf{114},  953--956 (1957)

\bibitem{LeshnoMLP}
Leshno, M., Lin, V.Y., Pinkus, A., Shocken, S.: Multilayer feedforward networks
  with a nonpolynomial activation function can approximate any function. Neural
  Networks  \textbf{6},  861--867 (1993). \doi{10.1016/S0893-6080(05)80131-5}

\bibitem{liu2018hierarchical}
Liu, H., Simonyan, K., Yang, Y.: Hierarchical representations for efficient
  architecture search. International Conference on Learning Representations
  (2018)

\bibitem{liu2019darts}
Liu, H., Simonyan, K., Yang, Y.: Darts: Differentiable architecture search.
  International Conference on Learning Representations  (2019)

\bibitem{cnnsgoogle}
Long, P.M., Sedghi, H.: Size-free generalization bounds for convolutional
  neural networks. CoRR  \textbf{abs/1905.12600} (2019),
  \url{https://arxiv.org/pdf/1905.12600v1.pdf}

\bibitem{Luo2018NeuralAO}
Luo, R., Tian, F., Qin, T., Liu, T.Y.: Neural architecture optimization. In:
  NeurIPS (2018)

\bibitem{MillerAndHedge}
Miller, G.F., Todd, P.M., Hegde, S.U.: Designing neural networks using genetic
  algorithms. Proc. 3rd Intl. Conf. Genetic Algorithms and Their Applications
  pp. 379--384 (1989)

\bibitem{Neto1997TuringUO}
Neto, J.P., Siegelmann, H.T., Costa, J.F., Araujo, C.P.S.: Turing universality
  of neural nets (revisited). In: Pichler, F., Moreno-D{\'i}az, R. (eds.)
  Computer Aided Systems Theory --- EUROCAST'97. pp. 361--366. Springer Berlin
  Heidelberg, Berlin, Heidelberg (1997). \doi{10.1007/BFb0025058}

\bibitem{MetaDesignOfFFNNs}
Ojha, V.K., Abraham, A., Sn\'{a}\v{s}el, V.: Metaheuristic design of
  feedforward neural networks: A review of two decades of research. Eng. Appl.
  Artif. Intell.  \textbf{60}(C),  97--116 (2017).
  \doi{10.1016/j.engappai.2017.01.013}

\bibitem{ComplexityTheoryNNs}
Orponen, P.: Computational complexity of neural networks: A survey. Nordic J.
  of Computing  \textbf{1}(1),  94--110 (1994)

\bibitem{Ostrand}
Ostrand, P.A.: Dimension of metric spaces and hilbert's problem 13. Bulletin of
  the American Mathematical Society  \textbf{71},  619–622 (1965).
  \doi{10.1090/S0002-9904-1965-11363-5}

\bibitem{ParkAndSandberg}
Park, J., Sandberg, I.W.: Universal approximation using radial-basis-function
  networks. Neural Computation  \textbf{3},  246--257 (1991).
  \doi{10.1162/neco.1991.3.2.246}

\bibitem{pmlr-v80-pham18a}
Pham, H., Guan, M., Zoph, B., Le, Q., Dean, J.: Efficient neural architecture
  search via parameters sharing. In: Dy, J., Krause, A. (eds.) Proceedings of
  the 35th International Conference on Machine Learning. Proceedings of Machine
  Learning Research, vol.~80, pp. 4095--4104. PMLR (10--15 Jul 2018)

\bibitem{PoggioTheoryOfNets}
Poggio, T., Girosi, F.: A theory of networks for approximation and learning.
  A.I. Memo No. 1140 (1989)

\bibitem{PoggioNetworks}
Poggio, T., Girosi, F.: Networks for approximation and learning. Proceedings of
  the IEEE  \textbf{78}(9) (1990). \doi{10.1109/5.58326}

\bibitem{RabinComputability}
Rabin, M.O.: Computable algebra, general theory and theory of computable
  fields. Trans. Amer. Math. Soc.  \textbf{95},  341--360 (1960).
  \doi{10.1090/S0002-9947-1960-0113807-4}

\bibitem{Real2017LargeScaleEO}
Real, E., Moore, S., Selle, A., Saxena, S., Suematsu, Y.L., Le, Q.V., Kurakin,
  A.: Large-scale evolution of image classifiers. In: Proceedings of the
  34$^{th}$ International Conference on Machine Learning (2017)

\bibitem{RogersComputability}
Rogers, Jr., H.: The Theory of Recursive Functions and Effective Computability.
  MIT Press, Cambridge, MA (1987)

\bibitem{ShaferRNNs}
Sch{\"a}fer, A.M., Zimmermann, H.G.: Recurrent neural networks are universal
  approximators. In: Proceedings of the $16^{th}$ International Conference on
  Artificial Neural Networks - Volume Part I. ICANN'06, vol.~27, pp. 632--640.
  Springer-Verlag, Berlin, Heidelberg (2006). \doi{10.1007/11840817{\_}66}

\bibitem{SchafferCaruana}
Schaffer, J.D., Caruana, R.A., Eshelman, L.J.: Using genetic search to exploit
  the emergent behavior of neural networks. Physics D  \textbf{42}(244-248)
  (1990). \doi{10.1016/0167-2789(90)90078-4}

\bibitem{SiegelAndXu}
Siegel, J.W., Xu, J.: On the approximation properties of neural networks. arXiv
  e-prints arXiv:1904.02311 (2019)

\bibitem{Siegelmann1991TuringCW}
Siegelmann, H.T., Sontag, E.D.: Turing computability with neural nets. vol.~4,
  pp. 77--80 (1991). \doi{10.1016/0893-9659(91)90080-F}

\bibitem{Siegelmann1995OnTC}
Siegelmann, H.T., Sontag, E.D.: On the computational power of neural nets. J.
  Comput. Syst. Sci.  \textbf{50},  132--150 (1995).
  \doi{10.1006/jcss.1995.1013}

\bibitem{SontagVC}
Sontag, E.D.: Vc dimension of neural networks. Neural Networks and Machine
  Learning p. 69–95 (1998)

\bibitem{stanley:naturemi19}
Stanley, K.O., Clune, J., Lehman, J., Miikkulainen, R.: Designing neural
  networks through evolutionary algorithms. Nature Machine Intelligence
  \textbf{1},  24–35 (2019)

\bibitem{StanleyEvolvingNNs}
Stanley, K.O., Miikkulainen, R.: Evolving neural networks through augmenting
  topologies. Evol. Comput.  \textbf{10}(2),  99--127 (Jun 2002).
  \doi{10.1162/106365602320169811}

\bibitem{SunEtAl}
Sun, Y., Yen, G.G., Yi, Z.: Evolving unsupervised deep neural networks for
  learning meaningful representations. IEEE Transactions on Evolutionary
  Computation  \textbf{23},  89--103 (2019). \doi{10.1109/TEVC.2018.2808689}

\bibitem{NIPS1988_149}
Tenorio, M.F., Lee, W.T.: Self organizing neural networks for the
  identification problem. In: Touretzky, D.S. (ed.) Advances in Neural
  Information Processing Systems 1, pp. 57--64. Morgan-Kaufmann (1989)

\bibitem{Valiant1984ATO}
Valiant, L.G.: A theory of the learnable. Commun. ACM  \textbf{27},  1134--1142
  (1984). \doi{10.1145/1968.1972}

\bibitem{Vapnik1995TheNO}
Vapnik, V., Chervonenkis, A.Y.: On the uniform convergence of relative
  frequencies of events to their probabilities. Theory of Probability and Its
  Applications  \textbf{16},  264--280 (1971).
  \doi{10.1007/978-3-319-21852-6{\_}3}, translated by B. Seckler

\bibitem{Vitushkin}
Vitushkin, A.: Some properties of linear superpositions of smooth functions.
  Dokl. Akad. Nauk SSSR  \textbf{156},  1258--1261 (1964)

\bibitem{NFLWolpert}
Wolpert, D.H., Macready, W.G.: No free lunch theorems for optimization. IEEE
  Transactions on Evolutionary Computation  \textbf{1}(1),  67--87 (1997).
  \doi{10.1109/4235.585893}

\bibitem{Wolpert2005CoevolutionaryFL}
Wolpert, D.H., Macready, W.G.: Coevolutionary free lunches. IEEE Transactions
  on Evolutionary Computation  \textbf{9},  721--735 (2005).
  \doi{10.1109/TEVC.2005.856205}

\bibitem{Wong:2018:TLN:3327757.3327928}
Wong, C., Houlsby, N., Lu, Y., Gesmundo, A.: Transfer learning with neural
  automl. In: Proceedings of the 32Nd International Conference on Neural
  Information Processing Systems. pp. 8366--8375. NIPS'18 (2018)

\bibitem{MetaOptimizationAlgoAnalysis}
Yang, X.S.: Metaheuristic optimization: Algorithm analysis and open problems.
  Proceedings of the $10^{th}$ International Symposium on Experimental
  Algorithms  \textbf{6630},  21--32 (2011).
  \doi{10.1007/978-3-642-20662-7{\_}2}

\bibitem{XinYao}
Yao, X.: Evolving artificial neural networks. Proceedings of the IEEE
  \textbf{87}(9) (1999). \doi{10.1109/5.784219}

\bibitem{Zoph2016NeuralAS}
Zoph, B., Le, Q.V.: Neural architecture search with reinforcement learning.
  CoRR  \textbf{abs/1611.01578} (2016)

\end{thebibliography}

\appendix
\section*{Appendices}

\section{PAC Is a Solver for FA}\label{app:asp_pac}

PAC learning, as defined by Valiant \cite{Valiant1984ATO}, is a slightly different problem than FA, as it concerns itself with whether a \emph{concept class} $C$ can be described with high probability with a member of a \emph{hypothesis class} $H$. It also establishes bounds in terms of the amount of samples from members $c \in C$ that are needed to learn $C$.
On the other hand, FA and its solution strategies concern themselves with finding a solution that minimizes the error, by searching through sequences of explicitly defined members drawn from a search space.

Regardless of these differences, PAC learning as a procedure can still be formulated as a solution strategy for FA. To do this, let $H$ be our search space.
Then note that the PAC error function $e_{pac}(h, c) = Pr_{x\sim\P}[h(x) \neq c(x)],\; c \in C,\;h \in H$, is equivalent to computing $\error_\sigma(h, c)$ for some subset $\sigma \subset dom(c)$, and choosing the frequentist difference between the images of the functions as the metric $d$.
Our objective would be to return the $h \in H$ that minimizes the approximation error for a given subset $\sigma \subset C$. Note that we do not search through the expanded search space $H^{\star, n}$.

Finding the right distribution for a specific class may be NP-hard \cite{BlumerLearnabilityVC}, and so $e_{pac}$ requires us to make certain assumptions about the distribution of the input values. Additionally, any optimizer for PAC is required to run in polynomial time. Due to all of this, PAC is a weaker approach to solve FA when compared to ASP, but stronger than ML since this solution strategy is fixed to the design of the search space, and not to the choice of function. Nonetheless, it must be stressed that the bounds and paradigms provided by PAC and FA are not mutually exclusive, either: the most prominent example being that PAC learning provides conditions under which the choice subset $\sigma$ is optimal.

With the polynomial constraint for PAC learning lifted, and letting the sample and search space sizes grow infinitely, PAC is effectively equivalent to ASP. However, that defies the purpose of the PAC framework, as its success relies on being a tractable learning theory.

\section{The VC Dimension and the Information Potential}\label{app:vc_dim}

There is a natural correspondence between the VC dimension \cite{BlumerLearnabilityVC,Vapnik1995TheNO} of a hypothesis space, and the information capacity of a sequence.

To see this, note that the VC dimension is usually defined in terms of the set of concepts (i.e., the input function $\F$) that can be shattered by a predetermined function $f$ with $img(f) = \{0,1\}$. It is frequently used to quantify the ability of a procedure to learn the input function $\F$.

In the FA framework we are more interested in whether the search space--also a set--of a given solution strategy is able to generalize well to multiple, unseen input functions. 
Therefore, for fixed $\F$ and $f$, the VC dimension and its variants provide a powerful insight on the ability of an algorithm to learn. 
When $f$ is not fixed, it is still possible to utilize this quantity to measure the capacity of a search space $\S$, by simply taking the union of all possible $f \in \allf{S}$ for a given $n$. 
However, when the the input functions are not fixed either, we are unable to use the definition of VC dimension in this context, as the set of input concepts is unknown to us. We thus need a more flexible way to model generalizability, and that is where we leverage the information potential $U(\S, n)$ of a search space.

\end{document}